\documentclass[final]{style/l4dc2025}

\usepackage{times}

\usepackage{multicol}

\usepackage{wrapfig}

\usepackage{bbm}
\usepackage{mathtools}
\usepackage{graphicx}
\usepackage{algorithmic}
\usepackage{booktabs}
\usepackage{adjustbox}
\usepackage{comment}
\usepackage{textcomp}
\usepackage{xcolor}
\usepackage{caption}
\usepackage{microtype}
\usepackage[ruled,algo2e]{algorithm2e}
\usepackage[title]{appendix}
\usepackage[symbol*]{footmisc}

\hypersetup{
    colorlinks=true,
    urlcolor=blue,
}

\def\BibTeX{{\rm B\kern-.05em{\sc i\kern-.025em b}\kern-.08em
    T\kern-.1667em\lower.7ex\hbox{E}\kern-.125emX}}

\newtheorem{assumption}{Assumption}
\newtheorem{problem}{Problem}

\newcommand*{\st}{\mathrm{s.t.}}
\newcommand*{\mbb}[1]{\ensuremath{\mathbb{#1}}}
\newcommand*{\mcal}[1]{\ensuremath{\mathcal{#1}}}
\newcommand*{\mbf}[1]{\ensuremath{\mathbf{#1}}}

\newcommand*{\appendixnote}[1]{Appendix~\ref{#1}}

\title{Safe, Out-of-Distribution-Adaptive MPC with Conformalized \ \\ Neural Network Ensembles}

\begin{document}
    \coltauthor{\Name{Jose Leopoldo Contreras} \Email{jcontr83@stanford.edu}\\
                \Name{Ola Shorinwa} \Email{shorinwa@stanford.edu}\\
                \Name{Mac Schwager} \Email{schwager@stanford.edu}\\
                \addr Department of Aeronautics and Astronautics,
        Stanford University, CA, USA
        }

    \maketitle

    \begin{abstract}
        We present SODA-MPC, a Safe, Out-of-Distribution-Adaptive Model Predictive Control algorithm that uses an ensemble of learned models for prediction with a runtime monitor to flag unreliable out-of-distribution (OOD) predictions. When an OOD situation is detected, SODA-MPC triggers a safe fallback control strategy based on reachability, producing a control framework that achieves the high performance of learning-based models while preserving the safety of reachability-based control. We demonstrate the method in the context of an autonomous vehicle, driving among dynamic pedestrians, where SODA-MPC uses a neural network ensemble for pedestrian prediction. We use the maximum singular value of the empirical covariance among the ensemble as the OOD signal for the runtime monitor. We calibrate this signal using conformal prediction to derive an OOD detector with probabilistic guarantees on the false-positive rate, given a user-specified confidence level. During in-distribution operation, the MPC controller avoids collisions with a pedestrian based on the trajectory predicted by the mean of the ensemble. When OOD conditions are detected, the MPC switches to a reachability-based controller to avoid collisions with the reachable set of the pedestrian assuming a maximum pedestrian speed, to guarantee safety under the worst-case actions of the pedestrian. %
        We verify SODA-MPC in extensive autonomous driving simulations in a pedestrian-crossing scenario. Our model ensemble is trained and calibrated with real pedestrian data, showing that our OOD detector obtains the desired accuracy rate within a theoretically-predicted range. We empirically show improved safety and task completion compared with two state-of-the-art MPC methods that also use conformal prediction but without OOD adaptation. Further, we demonstrate the effectiveness of our method with the large-scale multi-agent predictor Trajectron++, using large-scale traffic data from the nuScenes dataset for training and calibration.
    \end{abstract}

    
    \section{Introduction}
    Robotic autonomy stacks often leverage learning-based models for perception, trajectory prediction, and control. However, in many situations, the performance of these models depends highly on the distribution of the input data encountered at runtime. Learned models can exhibit strong performance when the runtime input data is similarly distributed to the training data (the \emph{in-distribution} setting), but performance suffers when the runtime inputs are significantly different from the training data (the \emph{out-of-distribution}, or OOD, setting\footnote{We specifically consider the OOD situation known as covariate shift, where the marginal distribution of the input data changes between training and runtime. In our problem, this arises from observed agents expressing unusual behavior not well-represented in the training data.}) \citep{nguyen2015deep}. This deterioration in performance in OOD settings has limited the utilization of deep-learning models in safety-critical applications. In this work, we introduce a \emph{Safe OOD-Adaptive} Model Predictive Controller (SODA-MPC), which enables robots to operate safely while sharing their task space with other agents that exhibit both in-distribution and OOD behavior.
    
    \begin{figure}[ht]
        \centering
        \includegraphics[width=0.48\linewidth]{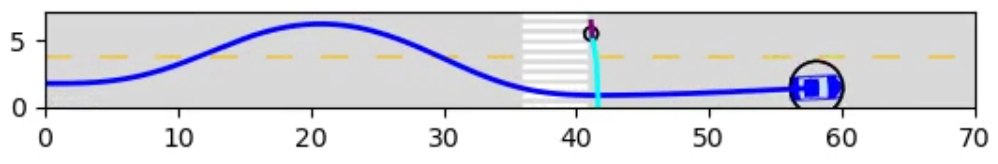}
        \includegraphics[width=0.48\linewidth]{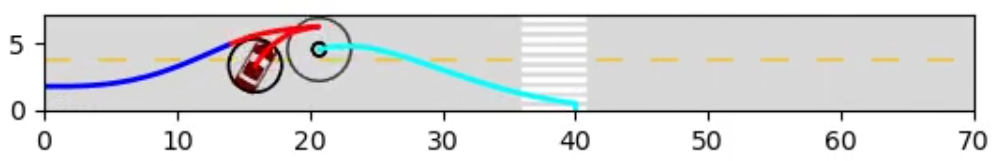}
        \caption{The \mbox{SODA-MPC} algorithm avoids pedestrian collisions in both the nominal (left) and adversarial (right) behavior settings. The algorithm correctly detects the adversarial pedestrian behavior as out-of-distribution, and switches its control strategy to the safe reachable-set-based MPC.}
        \label{fig:SAM_CP}
    \end{figure}

    We consider a standard autonomous vehicle (AV) control scheme, which uses an MPC with collision constraints to avoid other agents in the scene. The future trajectories of these agents (required for the collision avoidance constraint in the MPC) are predicted with a learned trajectory prediction model. We address a critical flaw in this architecture: such trajectory predictions are only reliable in in-distribution regimes. Our SODA-MPC algorithm alleviates this problem by introducing a statistically calibrated runtime monitor to detect OOD situations online and switch to a safe fallback controller in those cases. Our runtime monitor uses a neural network ensemble for OOD detection. However, we note that widely-used measures of disagreement for ensemble models are not statistically calibrated signals---they are not tied to an explicit probability of the data being in or out of distribution. Hence, we use conformal prediction to calibrate an OOD detector for the ensemble. When the detector flags an OOD scenario, the MPC ignores the neural network prediction and avoids colliding with the forward reachable sets of the agents, which is a guaranteed-safe fallback controller. The reachable sets are computed assuming a maximum speed bound for the agents.

    We demonstrate SODA-MPC in a scenario where an AV must avoid colliding with a pedestrian crossing the road, illustrated in Figure \ref{fig:SAM_CP}. We use a trajectory prediction model ensemble trained from real pedestrian data \citep{yang2019top}. To demonstrate the safety guarantees provided by our controller, we use a handful of held-out trajectories of the original data (which we refer to as \emph{nominal}), and create modified instances of these so that the pedestrian exhibits OOD behavior by running head-on towards the approaching car, a worst-case scenario for the AV.
    The car must avoid colliding with the pedestrian in both nominal and OOD regimes while attempting to reach a goal position on the road. In our evaluations, we find that our method reaches the goal position in a majority of cases and reliably detects OOD behavior such that it avoids all pedestrian collisions across all tests. We compare against two recent baselines that also use conformal prediction for AV control with learned pedestrian prediction models: SPDE \citep{lindemann2023safe}, which calibrates confidence intervals for the learned model, and MARC \citep{muthali2023multi}, which calibrates reachable sets for the learned models. Neither method directly addresses OOD behavior. We find that our algorithm is safer than SPDE in OOD regimes and more efficient than MARC which tends to be overly conservative in in-distribution regimes, as demonstrated in  Figure~\ref{fig:sim_results}. In addition, we demonstrate the success of our detector in identifying OOD scenarios when applied to the large-scale trajectory predictor Trajectron++ \citep{salzmann2020trajectron++}.
    
    Our contributions are as follows:
    \begin{itemize}
        \item We derive an MPC architecture, SODA-MPC, for learned predictive models that safely adapts to out-of-distribution model predictions. Our MPC switches to a guaranteed-safe reachable set based prediction in out-of-distribution regimes.
    
        \item In addition, we introduce an OOD runtime monitor with a guaranteed error rate using neural network ensembles. The monitor is statistically calibrated with conformal prediction, supplying the switching signal for our SODA-MPC.
        
    \end{itemize}

    \section{Related Works}
    \subsection{Out-of-Distribution Detection in Neural Networks}
    Methods for OOD detection can be broadly grouped into training-free and fine-tuning methods. Training-free methods directly leverage the outputs \citep{liu2020energy, wang2021can, hendrycks2016baseline}, e.g., softmax probabilities, and activation patterns \citep{sun2021react, dong2022neural, sun2022dice, sharma2021sketching} of a trained model to detect OOD inputs.
    Fine-tuning methods modify the underlying architecture of the model, including the loss function, to estimate the network's confidence in its predictions to distinguish between in-distribution and out-of-distribution settings, e.g., \citep{devries2018learning, lee2017training, hein2019relu, thulasidasan2019mixup, madras2019detecting}. We refer interested readers to \citep{yang2021generalized} for a more detailed discussion. In the context of robotics, \cite{sinhaRSS24LLM_OOD} detects visual anomalies with an LLM, and enacts a safe fallback MPC if a danger is detected. Although these existing methods work well in many situations, they do not provide provable OOD detection probabilistic guarantees, in general.

    \subsection{Conformal Prediction in Trajectory Prediction and Motion Planning}
    Conformal prediction \citep{vovk2005algorithmic, shafer2008tutorial} provides a black-box, distribution-free method for generating prediction regions from online observations, making it ideal for wide-ranging applications, e.g., in automated decision support systems \citep{straitouri2023improving, babbar2022utility} and in  safe trajectory prediction and motion planning \citep{luo2022sample, lindemann2023safe, dixit2023adaptive, lekeufack2023conformal, sun2023conformal}.
    In trajectory prediction, some existing methods leverage conformal prediction for the detection of unsafe situations in early-warning systems \citep{luo2022sample}, verification of autonomous systems \citep{fan2020statistical, dietterich2022conformal}, and motion planning \citep{chen2021reactive, lindemann2023safe, muthali2023multi}. In our work, we utilize conformal prediction to provide probabilistic guarantees on the safety of a trajectory prediction module and model predictive controller for autonomous navigation of an agent. Similar to out work, \citep{lindemann2023safe} and \citep{muthali2023multi} apply conformal prediction to learned trajectory prediction models to obtain in-distribution safety guarantees for an MPC. In contrast, we use conformal prediction and a safe fallback controller to obtain safety for both in-distribution and OOD settings.

    \section{Problem Formulation}
    We consider an autonomous robot operating in an environment with other agents, where the dynamics model of the robot is given by: ${\mathbf{x}_{t+1} = f_e(\mathbf{x}_t,\mathbf{u}_t),}$
    where $\mathbf{x}_t \in \mathcal{X} \subseteq \mathbb{R}^{n_e}$ represents the state of the robot at time step $t$, $\mathcal{X}$ represents its workspace, $\mathbf{x}_0 \in \mathcal{X}$ represents its initial state, ${\mathbf{u}_t \in \mathcal{U} \subseteq \mathbb{R}^{m}}$ denotes its control input at $t$, where $\mathcal{U}$ represents the set of permissible control inputs, and ${f_e: \mathbb{R}^{n_e} \times \mathbb{R}^{m} \to \mathbb{R}^{n_e}}$ describes the robot's dynamics, which are presumed known.
    
    The robot seeks to navigate to a goal location optimally without colliding with other agents in its environment. Our method can handle multiple or single agents. We do not assume that these dynamic agents cooperate with the robot, allowing for potential adversarial behaviors.
    Let the random variable $Y$ denote the joint state of multiple agents, which follows an unknown stationary distribution $\mathcal{D}$ over the states of agents, i.e., ${Y \sim \mathcal{D}}$, where ${Y_t \in \mathbb{R}^{n_a}}$ denotes the state of the agents at time $t$, and $n_a$ is the stacked state dimension of all the agents.
    We use $y$ when referring to an observation of one of these multi-agent states, with $y_t$ referring to the agents' state at time $t$. The true dynamics of these agents are unknown, but can be approximated with a learned model $f_a: \mathbb{R}^{n_a} \times \mathbb{R}^m \to \mathbb{R}^{n_a}$, learned from a dataset $D$ where $D := \{y^{(1)},\dots,y^{(k)}\}$, and $y^{(i)}$ represents the $i$th complete observed joint state of the agents in the environment.
    
    \begin{assumption}
        We have access to $k$ independent realizations $y$ of the distribution $\mathcal{D}$, collected in the dataset ${D:=\{y^{(1)},\dots,y^{(k)}\}}$.
    \end{assumption}
    
    This assumption is not limiting in practice, but is essential for the theoretical guarantees provided by our method. 
    Now, we provide a formal statement of the problem.
    
    \begin{problem}[Safe Adaptive Controller]
        \label{prob:safe_adaptive_controller}
        Given the robot's dynamics 
        and a set of observed in-distribution trajectories of agents $D:=\{y^{(1)},\dots,y^{(k)} \}$, design an MPC controller that enables an autonomous robot to safely navigate its environment, when operating alongside other agents whose trajectories may be ``in-distribution" or ``out-of-distribution."
    \end{problem}

    \section{Model Predictive Control}

    In this work, we utilize a learned ensemble, consisting of $n$ individual multi-layer perceptron (MLP) neural networks for trajectory prediction, given a dataset $D$. Neural network ensembles enable estimation of the probability distribution of the outputs of the model, which is valuable in preventing overfitting and in confidence estimation \citep{gustafsson2020evaluating}. We note that approaches such as Bayesian neural networks \citep{jospin2022hands} and dropout methods for neural networks \citep{labach2019survey} can be considered as special examples of ensemble learning, providing similar advantages. We use relatively small MLPs as the individual models in the ensemble to demonstrate the effectiveness of our algorithm with simple predictors. 
    
    Given the neural network ensemble, we summarize the predictions of the individual models within the ensemble using the mean and the covariance of the predictions, given by:
    \begin{align}
        {\mathbf{s}}_{t+1} =  \frac{1}{n}\sum_{i=1}^n f_{a}^{i}(\mathbf{s}_{t- \ell:t}), \quad %
        \Sigma_{t+1} = \frac{1}{n-1}\sum_{i=1}^n (f_{a}^{i}(\mathbf{s}_{t- \ell:t}) -{\mathbf{s}}_{t+1})(f_{a}^{i}(\mathbf{s}_{t- \ell:t}) - {\mathbf{s}}_{t+1})^\intercal, \label{eq:covariance_pred}
    \end{align}
    where $f_{a}^{i}$ represents the $i$th model in the ensemble, ${\mathbf{s}}_{t+1}$ denotes the mean prediction of the state of the environment at time step $t + 1$ (with the prediction made at time $t$), and
    ${\Sigma}_{t+1}$ represents the unbiased estimate of the corresponding covariance matrix for the predictions made by the models in the ensemble at time step $t$. The input to each network is $\mathbf{s}_{t- \ell:t}$, a concatenation of the (observed or predicted) states of the agents over the preceding ${\ell}$ steps, including the current time step. We write ${\mathbf{s}}_{t+1} = f_{a}(\mathbf{s}_{t- \ell:t})$ to denote the ensemble mean.

    We begin with designing a controller for the nominal case, where the trajectories of the dynamic agents follow the distribution given by $\mcal{D}$. In this case, we design a model predictive controller leveraging the predictive model $f_{a}$ to enforce collision-avoidance constraints. The corresponding model predictive control (MPC) problem is given by:
    \begin{equation}
        \begin{aligned}
            \textit{MPC\ I:} \hspace{0.5em} \min_{\mathbf{x}_{1:T}, \mathbf{u}_{0:T-1}} \;& J(\mathbf{x}_{1:T},\mathbf{u}_{0:T - 1})\\
            \st \;&\mathbf{x}_{\tau+1} = f_e(\mathbf{x}_\tau,\mathbf{u}_\tau), \enspace
            \mathbf{s}_{\tau+1} = f_a(\mathbf{s}_{\tau - l:\tau}), \enspace \forall \tau = 0,\dots, T-1, \\
            &g(\mathbf{x}_\tau,\mathbf{s}_\tau) \leq 0, \enspace \forall \tau=1,\dots,T, \label{eq:MPC-I}
        \end{aligned}
    \end{equation}
    where $J: \mathbb{R}^{T \cdot n_e} \times \mathbb{R}^{T \cdot m} \to \mathbb{R}$ denotes the objective function and $\mathbf{x}_{1:T}$ and  $\mathbf{u}_{0:T - 1}$ denote the concatenation of the optimization variables representing the robot's states and control inputs over the pertinent time steps respectively. We assume we have access to $s_{0}$, from which we predict the trajectory $\mathbf{s}_{\tau - l:\tau}$,~${\forall \tau}$, and $\mathbf{x}_{0}$, the initial state of the robot.
    We assume that the MPC problem has a planning horizon of $T$ time steps. We specify constraints enforcing initial conditions, collision avoidance, and other undesirable interactions in $g: \mathbb{R}^{n_e} \times \mathbb{R}^{n_a} \to \mathbb{R}$. The MPC problem is re-solved every $H$ time steps with updated observations.
    
    We note that the resulting MPC-based controller \eqref{eq:MPC-I} fails to provide safety guarantees, particularly when the dynamic agents exhibit OOD behavior. To address this limitation, we develop an adaptive controller, considering both in-distribution and OOD settings. Our OOD controller leverages the reachable set of the dynamic agents, representing the set of states that the dynamic agents can reach over a specified time duration. We denote the reachable set for the agents in the environment at time step $t$ by $\mathcal{R}_t \subseteq \mathcal{S}$, which we compute from: ${\mathcal{R}_{t+1} = \textsc{Reach}(\mathcal{R}_t)}$,
    with $\mathcal{R}_0 = \{\mathbf{s}_0 \}$, where ${\textsc{Reach}: \mathbb{S} \to \mathbb{S}}$ and $\mathbb{S}$ denotes all subsets of $\mathbb{R}^{n_a}$. 
    In computing the reachable set of the dynamic agents, e.g., via velocity-based reachability analysis, we make the following assumption:
    \begin{assumption}
        We assume knowledge of a maximum speed ${v_{max} \in \mbb{R}}$ for all agents.
    \end{assumption}
    
    This assumption is realistic, as we can use reasonable bounds for the top speed of agents in the robot's workspace, such as pedestrians.
    To provide safety assurances in OOD settings, we consider an MPC problem where the predictive model for the state of the agents in \eqref{eq:MPC-I} is replaced with a reachable-set-based constraint. The corresponding MPC problem is given by:
    \begin{equation}
        \begin{aligned}
            \textit{MPC\ II:} \hspace{0.5em} \min_{\mathbf{x}_{1:T}, \mathbf{u}_{0:T - 1}} \;& J(\mathbf{x}_{1:T},\mathbf{u}_{0:T - 1})\\
            \st \;&\mathbf{x}_{\tau+1} = f_e(\mathbf{x}_\tau,\mathbf{u}_\tau), \enspace 
            \mathcal{R}_{\tau+1} = \textsc{Reach}(\mathcal{R}_\tau), \enspace \forall \tau = 0,\dots, T-1, \\
            &g(\mathbf{x}_\tau,\mathcal{R}_\tau) \leq 0, \enspace \forall \tau=1,\dots,T, \label{eq:MPC-II}
        \end{aligned}
    \end{equation}
    where we overload notation, letting $g: \mathbb{R}^{n_e} \times \mathbb{S} \to \mathbb{R}$ denote constraints preventing undesirable interactions 
    (e.g., collisions) defined over subsets of $\mathbb{R}^{n_a}$. Although the \emph{MPC II} controller may be quite conservative, this controller can ensure safety even in situations where agents' behaviors do not follow the distribution $\mcal{D}$.
    
    To design a safe, OOD-adaptive controller for Problem~\ref{prob:safe_adaptive_controller}, we compose the MPC-based controller for the nominal setting (\emph{MPC I}) with the reachable-set-based controller (\emph{MPC II}) to enable safe navigation by autonomous robots, necessitating the design of a rule for switching between both controllers. Specifically, an autonomous robot must be able to distinguish between dynamic agents which are acting \emph{nominally} from those agents with \emph{OOD} behavior, concepts which we precisely define later in this work. To address this challenge, we introduce an OOD detector, which leverages conformal prediction to provide probabilistic guarantees. We state the OOD detection problem, prior to discussing our proposed approach to solving this problem:
    
    \begin{problem}[OOD detection]
        Given the set of observed trajectories $D:=\{y^{(1)},\dots,y^{(k)} \}$ and a failure probability $\delta \in (0,1)$, identify a signal $\rho$ that can correctly determine, with a probability of $1-\delta$, that an observed trajectory is sampled from the same probability distribution as that which generated the trajectories in set $D$.
    \end{problem}

    \section{Out-of-Distribution Detection}
    We utilize the spectral norm $\rho_{t+1}$ of the empirical covariance matrix ${\Sigma}_{t+1}$ \eqref{eq:covariance_pred} as our uncertainty measure to detect OOD samples, with ${\rho_{t+1} = \|{\Sigma}_{t+1} \|_2}$, noting that $\Sigma_{t+1}$ is positive semidefinite and symmetric for all $t$. However, related measures such as the trace, spectral norm, and Frobenius norm of the covariance matrix and techniques such as \citep{sharma2021sketching} and \citep{luo2022recencyprediction}
    can also be used as the uncertainty measure.
    We note that the techniques in  \citep{sharma2021sketching} and \citep{luo2022recencyprediction} require access to the whole training dataset, while we only require a handful of held-out data samples.
    We leverage conformal prediction to provide provable probabilistic safety guarantees for the safe adaptive controller introduced in this work (see 
    \appendixnote{app:conformal_prediction}
    for a brief introduction of conformal prediction).

    To define a valid prediction region, we use the spectral norm of the empirical covariance matrix
    as the nonconformity score: larger values of $\rho$ signify greater degrees of nonconformity. 
    The following remark results from Lemma $1$ in \citep{tibshirani2019conformal}, provided that $\rho^{(i)}$ for $i = 1,\dots,N$, the noncomformity measures of the data points in $D_{\mathrm{cal}}$, are exchangeable, with $\rho_{t + 1}$ being the nonconformity score associated with the prediction of the pedestrian position at time step $t+1$.
    \begin{remark}
        \label{rem:prob_guarantee_marginal}
        Given a failure probability ${\delta \in (0, 1)}$, a calibration dataset ${D_{\mathrm{cal}} \subset D}$, and the predictions $f_{a}^{i}(\cdot)$,~${\forall i}$, we have: ${P(\rho_{t + 1} \leq C) \geq 1 - \delta},$
         for ${t > 0},$ where $C$ represents the prediction region associated with $\delta.$
    \end{remark}
    
    Remark~\ref{rem:prob_guarantee_marginal} states a standard result from Conformal Prediction, which allows strong guarantees despite its simplicity. 
    For example, with $10$ calibration data points, we can obtain a test with an exact 10\% error rate. To understand this more easily, note that this result considers both the calibration data and the new unseen data points as random variables. In other words, this error rate is obtained exactly when marginalizing over all calibration sets drawn i.i.d. from the data distribution. However, in an engineering application, the calibration dataset is sampled once and fixed, which results in a probability distribution over the obtained error rate for the OOD classifier, which requires a more sophisticated statistical analysis \citep{shafer2008tutorial, angelopoulos2021gentle} beyond the scope of this paper. We state the distribution of coverage conditioned on the calibration set in the following theorem.
    
    \begin{theorem}
        \label{thm:prob_guarantee}
        Conditioned on a calibration dataset ${D_{\mathrm{cal}} \subset D}$, the coverage achieved by conformal prediction follows an analytic distribution given by:
        \begin{equation}
            P(\rho_{t + 1} \leq C) \sim \mathrm{Beta}(N + 1 - K, K), \enspace \forall t,
        \end{equation}
        where $N$ is the size of the $D_{cal}$ dataset and $K$ is the index value of the nonconformity score used to set $C$ when the scores in $D_{cal}$ are placed in nondecreasing order.
    \end{theorem}
    
    \begin{proof}
        The proof is presented in \citep{vovk2012conditional}. Hence, we omit the proof here.
    \end{proof}
    
    Theorem~\ref{thm:prob_guarantee} indicates that the error rate follows a Beta distribution, which can be used to obtain a confidence that the test holds with a given error rate. E.g., we may calibrate for a $10\%$ error rate, but find from this Beta distribution that we will obtain an error rate less than or equal to $10\%$  with $60\%$ probability and at most a $12\%$ error rate with $90\%$ probability. We note that with more samples in our calibration set, we obtain a tighter Beta distribution, resulting in a greater likelihood that we obtain an error rate close to the calibrated value. 
    This more nuanced analysis is discussed in \appendixnote{app:sim_appendix_stats_analysis},
    where we calibrate for a $4\%$ error rate, and obtain a $4.4\%$ error rate for our specific calibration set.
    At runtime, we detect OOD behavior when the non-conformity score $\rho_{t + 1}$ exceeds the calibrated threshold $C$.
    
    We note that \emph{naive generation} of the calibration dataset would result in the violation of the exchangeability assumption, since trajectories of the dynamic agents collected within the same interaction between the agents and the autonomous robot are not independent. In particular, such trajectories are time-correlated, making them non-exchangeable. To address this issue,
    we form the calibration dataset $D_{\mathrm{cal}}$ by uniformly randomly removing one data point from each trajectory in $D_{\mathrm{train}}$ and placing the sampled data in the calibration set. We use the remaining datapoints in the trajectory for training the neural networks. We provide a more detailed discussion of the data-generation procedure in \appendixnote{app:sim_appendix_dataset}.
    We note that the resulting calibration dataset satisfies the exchangeability assumption, by satisfying the stronger assumption that trajectories are independent and identically distributed. To complete the calibration procedure in the conformal prediction framework, we compute the noncomformity measure associated with each data point in $D_{\mathrm{cal}}$ and place them in nondecreasing order, to be used in the specification of $C$ based on the quantile of the empirical distribution corresponding to the desired value of $\delta$.

\section{Safe, Out-of-Distribution-Adaptive MPC}
    Given an OOD detector, we compose a safe OOD-adaptive control architecture that uses the MPC-based controller in \eqref{eq:MPC-I} with the nominally-acting, dynamic agents, and switches over to a conservative reachable-set-based approach when OOD behavior is detected. Algorithm~\ref{alg:safe_adaptive_controller} summarizes our proposed method \emph{SODA-MPC}: Safe, Out-of-distribution-Adaptive Model Predictive Control, which is illustrated in Figure~\ref{fig:controller_structure} and described in greater detail in \appendixnote{app:soda_mpc}.

    \section{Simulations} \label{Simulations}
    We evaluate the performance of our controller in both in-distribution and OOD settings, and compare its performance to that of controllers presented in \citep{muthali2023multi} and \citep{lindemann2023safe}, which utilize conformal prediction to provide probabilistic safety guarantees, in a problem where 
    an autonomous vehicle attempts to safely navigate to its destination without colliding with a pedestrian crossing the road at a crosswalk. In the in-distribution case, which we refer to as \emph{nominal} behavior, the pedestrian follows a trajectory from the test dataset, $D_{\mathrm{test}}$. Meanwhile, in the OOD case referred to as \emph{insurance fraud} behavior, the pedestrian attempts to force a collision by approaching the vehicle at maximum speed, which constitutes a worst-case scenario, after initially acting nominally for $1.3$ seconds. To overcome such an attack, an autonomous vehicle must detect the OOD behavior quickly and adapt its controller to avoid a collision.

    \begin{minipage}[t]{0.46\textwidth}
            \vspace{0ex}
            \begin{algorithm2e} [H]
                \caption{SODA-MPC: Safe, \\ Out-of-distribution-Adaptive \\ Model Predictive Control}
                \label{alg:safe_adaptive_controller}
            
                \SetKwRepeat{doparallel}{do in parallel}{while}
            
                \KwIn{Calibraton Dataset $D_{\mathrm{cal}}$ and Failure Probability $\delta$ or Index $K$.}
            
                Calibrate the OOD Detector.
            
                \For{$t \gets 0,H,2H,\ldots$} {
                    \tcp{Observe the environment}
                    $y_{t} \gets$ Sensor(t)\;
                    \tcp{Detect OOD behavior}
                    $B \gets$ OOD\_Detector($y_{t}$)\;
                    \tcp{Execute the MPC Controller.}
                    \uIf{$B = 0$} {
                        \tcp{Nominal Controller.}
                        $(u_{t},\ldots,u_{t + T - 1}) \gets$ MPC I \eqref{eq:MPC-I}
                    }
                    \uElse {
                        \tcp{Reachable-Set-Based Controller.}
                        $(u_{t},\ldots,u_{t + T - 1}) \gets$ MPC II \eqref{eq:MPC-II}
                    }
                    Apply $u_{t:t+H}$.
                }
            \end{algorithm2e}
            \end{minipage}
            \hfill
            \begin{minipage}[t]{0.46\textwidth}
                \vspace{0ex}
                \includegraphics[width=0.95\linewidth]{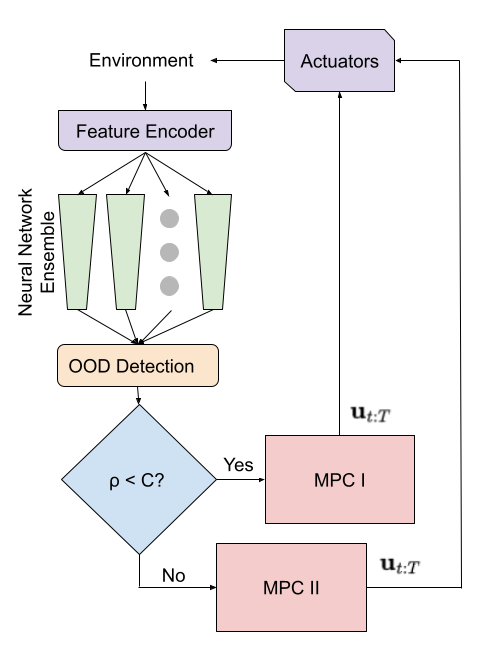}
                \captionof{figure}{Architecture of SODA-MPC: SODA-MPC detcts OOD behavior from online observations to select a safe control strategy, depending on the behavior of other agents.}
                \label{fig:controller_structure} 
            \end{minipage}
            \vspace{2ex}

    \subsection{Small, Feedforward MLP Trajectory-Prediction Models}
    Here, we examine the effectiveness of SODA-MPC in detecting OOD behavior, in problems where the trajectory-prediction model consists of small, feedforward MLPs. Specifically, we predict the future state of a pedestrian using an ensemble of identical MLPs, each with 66 parameters in total, trained on the VCI dataset \citep{yang2019top}. We present the architecture of these models and the associated training dataset in 
    \appendixnote{app:sim_appendix}.

    \smallskip
    \noindent\textbf{Empirical Results.}
    We assess the performance of our control algorithm \emph{\mbox{SODA-MPC}} against four baselines, namely: (a) \emph{SPDE} \citep{lindemann2023safe}; (b) \emph{MARC} \citep{muthali2023multi}; (c) \emph{Reachable Sets Only}, a variant of \mbox{SODA-MPC} utilizing only the MPC II controller; and (d) \emph{Ensembles Only}, a variant of \mbox{SODA-MPC} utilizing only the MPC I controller. We examine each algorithm in $20$ simulations, presenting the results achieved by each algorithm in the subsequent discussion. Each simulation terminates in one of three ways: (a) the autonomous vehicle goes past the pedestrian while avoiding collisions, (b) the autonomous vehicle does not go past the pedestrian but avoids a collision (e.g., when stopping to allow the pedestrian to finish crossing the road), or (c) the autonomous vehicle collides with the pedestrian. In \mbox{SODA-MPC}, we utilize a failure probability of approximately $4\%$ in the conformal prediction procedure for computing prediction regions, associated with ${C = 0.012}$. In contrast, the uncertainty quantification for the MARC algorithm \citep{muthali2023multi} corresponds to a failure probability of $5\%$ and the ball-shaped confidence areas for the SPDE algorithm \citep{lindemann2023safe} correspond to a failure probability of $5\%$. In all the figures, we depict the past trajectory of the autonomous vehicle in blue when the autonomous vehicle detects nominal behavior and in red when the autonomous vehicle detects OOD behavior. We depict the past trajectory of the pedestrian in cyan.
    In Figure~\ref{fig:sim_results}, we show that \mbox{SODA-MPC} and the more conservative methods MARC and \emph{Reachable Sets Only} do not collide with the pedestrian, even when the pedestrian attempts to force a collision. In the nominal case, the \emph{Ensembles Only} method achieves the highest success rate at $70\%$, followed jointly by \mbox{SODA-MPC} and SPDE, which achieve a success rate of $60\%$. Consequently, \mbox{SODA-MPC} provides a desirable tradeoff across both behavior modes.
    
    \begin{figure}[th]%
        \centering
        \begin{minipage}[t]{0.47\textwidth}
            \vspace{0pt}
            \centering
            \includegraphics[width=\linewidth]{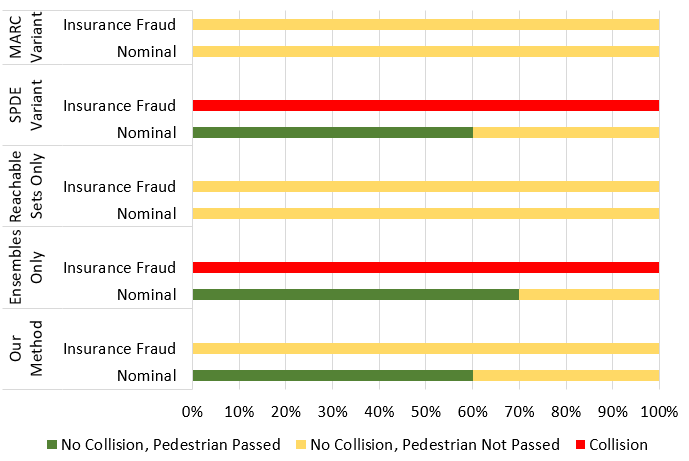}
            \caption{In a pedestrian-crossing scenario, only \mbox{SODA-MPC} (our method) both passed the horizontal position of the pedestrian in the nominal case and avoided collisions with the pedestrian in the \emph{insurance fraud} case, compared to the other baselines.
            \label{fig:sim_results}}
        \end{minipage}
        \hfill
        \begin{minipage}[t]{0.47\textwidth}
            \vspace{0pt}
            \centering
            \includegraphics[width=\linewidth]{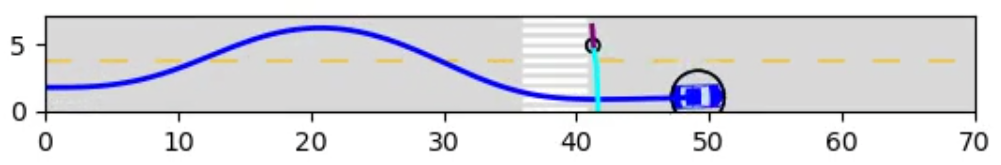}
            \includegraphics[width=\linewidth]{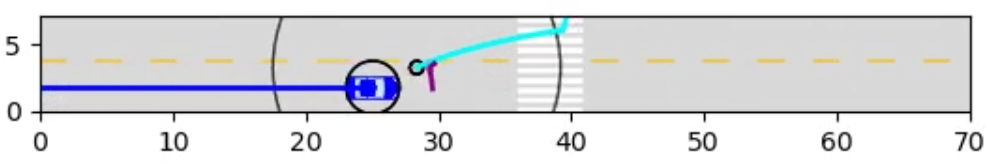}
            \includegraphics[width=\linewidth]{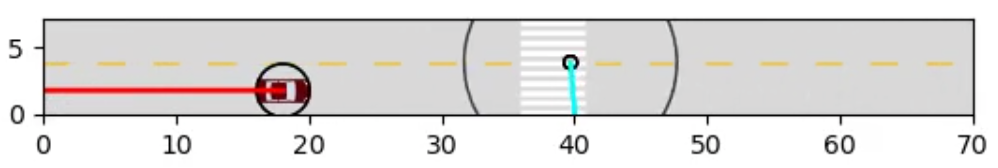}
            \includegraphics[width=\linewidth]{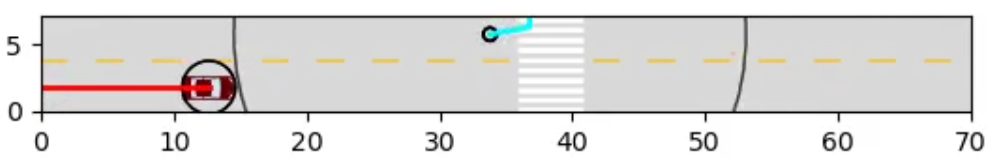}
            \caption{In the top-two rows, the \emph{Ensembles Only} method avoids collisions in the nominal case (top) and collides in the \emph{insurance fraud} setting (bottom). In the bottom-two rows, the overly-conservative \emph{Reachable Sets Only} method avoids collisions in both settings.}
            \label{fig:ensemble_only}
        \end{minipage}
    \end{figure}

    While the \emph{Ensembles Only} method could reliably avoid collisions in the nominal case, the method failed to avoid collisions in the \emph{insurance fraud} case, since the predictions of the pedestrian's trajectory differed wildly from the actual trajectory taken by the pedestrian, as depicted in Figure~\ref{fig:ensemble_only}. In contrast, the \emph{Reachable Sets Only} method avoids collisions in all settings; however, the autonomous vehicle was consistently unable to reach its desired goal location during the simulation, as illustrated in Figure~\ref{fig:ensemble_only}. We note that the MARC algorithm achieved similar conservative results to those of the \emph{Reachable Sets Only} method.
    Although SPDE remains safe while navigating the vehicle to its goal in the nominal case, SPDE does not avoid collisions when the pedestrian attempts to forcefully collide with the vehicle: this can be attributed to the erroneous predictions generated by the neural network ensemble. Though the method adds a buffer around the position of the pedestrian based on the estimated uncertainty of its prediction, the uncertainty is calibrated under the assumption that the true pedestrian trajectory is drawn from the same distribution as in the training data, while in the \emph{insurance fraud} setting, it is drawn from an altogether different distribution. In contrast, \mbox{SODA-MPC} avoids collisions in all settings, e.g., in the \emph{insurance fraud} setting, the controller switches from the nominal control strategy to the reachable-set-based control strategy after detecting OOD behavior, depicted in Figure~\ref{fig:SAM_CP}.
    We provide empirical results on the statistical guarantees of our OOD detector in 
    \appendixnote{app:sim_appendix_stats_analysis},
    showing that our detector achieves a false-positive rate of 4.4\% while correctly identifying OOD behavior 91.3\% of the time.

    \subsection{Trajectron++, a Large-Scale Trajectory-Prediction Model}
    Here, we evaluate the performance of our proposed OOD detector applied to the large-scale trajectory predictor Trajectron++ \citep{salzmann2020trajectron++} as the trajectory predictor on the nuScenes \citep{caesar2020nuscenes} dataset. We provide additional setup details in 
    \appendixnote{app:sim_appendix_trajectron}.

    \smallskip
    \noindent\textbf{Empirical Results.}
    For the calibration procedure, we set $K=97$ for $N=100$ for the calibration set point, yielding a detection threshold of $C=0.933$. Here, predictions are evaluated at the trajectory level, rather than at individual time steps: a trajectory as a whole is considered to be anomalous if the uncertainty value for at least one of its constituent time steps is greater than the detection threshold, and is considered to be nominal otherwise.
    Nominal cases cover 6230 trajectories of pedestrians in the dataset, and OOD cases use 230 trajectories with shifted pedestrian behavior. We use a calibrated error rate of $\delta \approx 4\%$ for individual data points. In this setting, our OOD detector correctly identifies 5361 (86.1\%) of the nominal trajectories as nominal, with a false-positive rate of 13.9\% (869 trajectories). Likewise, our OOD detector correctly identifies 211 (91.7\%) of the \emph{insurance fraud} trajectories as OOD, with a false-negative rate of 8.3\% (19 trajectories), summarized in Table \ref{tab:confusion_matrix_trajectron}
    in 
    \appendixnote{app:sim_appendix_trajectron}.
    Compared to an ensemble of identically trained MLP models, the increase in uncertainty seen between different modes when an agent was displaying anomalous behavior was still detectable but less pronounced: we believe this may be ascribed to the cost function of the Trajectron++ model being specifically designed to promote divergence between the prediction modes, rather than training them all to react identically to training data points. This phenomenon was counteracted by evaluating predictions at the trajectory level, which increases the overall sensitivity of the system to detections of anomalous data points. As trajectories displaying \textit{insurance fraud} behavior still have a large number of data points below the detection threshold, this vastly improved the detection rate for anomalous behavior, although the increased sensitivity did increase the incidence of nominal data points being mischaracterized as anomalous compared to the experiments with ensembled feedforward MLP models.

    \section{Conclusion, Limitations, and Future Work}
    We introduce SODA-MPC, a safe, adaptive controller that enables autonomous robots to safely navigate in their environments in nominal or OOD settings. SODA-MPC is comprised of a nominal controller based on learned trajectory predictors for in-distribution settings, and safe-by-design controller based on reachable sets for OOD cases. Our OOD detector uses conformal prediction to provide provable probabilistic guarantees on the validity of the prediction regions. We demonstrate the safety and efficacy of our controller in a pedestrian-crossing scenario. With SODA-MPC, the autonomous vehicle avoids collisions with the pedestrian in all settings, and is able to pass the position of the pedestrian when the pedestrian acts in-distribution; whereas some other existing methods either result in collisions, when the pedestrian attempts to force them, or are overly conservative (failing to reach the goal).
    Our OOD-controller may be too conservative in some settings, especially in dense (congested) scenarios. In future work, we seek to explore the utilization of deep-learned controllers in combination with (or in lieu of) MPC-based controllers in our control framework, which could enable the development of more expressible control frameworks. Further, we are interested in exploring the use of our method in uncertainty quantification to identify and search for areas of the predictor input space not appropriately covered by training data, and in examining other uncertainty quantification metrics for OOD detection. 
    
\acks{This work was supported in part by ONR grant N00014-23-1-2354, NASA ULI grant 80NSSC20M0163, DARPA grant HR001120C0107, and NSF grant FRR 2342246. Toyota Research Institute provided funds to support this work.}

    \bibliography{references}
    
    \clearpage

    \appendix

    \section{Conformal Prediction}
    \label{app:conformal_prediction}
    Conformal prediction was introduced in \citep{vovk2005algorithmic} and \citep{shafer2008tutorial} to compute prediction regions for complex predictive models without making assumptions about the distribution underlying the predictive model \citep{angelopoulos2021gentle}. We present a brief introduction based on the assumption that the observations arise from independent and identical distributions for easier presentation. However, we note that the theory of conformal prediction only requires the assumption of \emph{exchangeability}. Observations in a set ${\{\rho^{(1)},\ldots,\rho^{(N)}\}}$ are exchangeable if all possible orderings of these observations are equally likely to occur. Consequently, exchangeability represents a weaker assumption compared to the assumption of independent, identically-distributed observations.
    
    Let $R^{(1)},\dots,R^{(N)}$ be $N$ independent and identically distributed random variables. The variable $R^{(i)}$ is generally referred to as the \textit{nonconformity measure} or \emph{score}. Conformal prediction enables us to construct a valid prediction region $C$ associated with a given failure probability ${\delta \in (0,1)}$, satisfying: ${P(R^{(0)} \leq C) \geq 1-\delta},$
    for a random variable $R^{(0)}$. If the set of observations $R^{(1)},\dots,R^{(N)}$ is sorted in non-decreasing order, with the addition of $\infty$ to the set, we can define a valid prediction region
    by setting $C$ to be the $(1-\delta)$th quantile, given by ${C = R^{(p)}}$, where ${p := \lceil (N+1)(1-\delta) \rceil}$ and ${\lceil \cdot \rceil}$ denotes the ceiling function (Lemma $1$ of \citep{tibshirani2019conformal}). We note that $p$ represents the index in the sorted set of observations.
    
    \section{OOD-Adaptive Controller Design}
    \label{app:soda_mpc}
    We calibrate the OOD detector prior to solving the MPC problems for the control inputs, given a calibration dataset $D_{\mathrm{cal}}$ and a desired failure probability $\delta$, or equivalently, the corresponding index of the sample in the sorted set of samples of nonconformity scores. Here, we describe the calibration process given the value of this index. 
    Given $N$ samples in the sorted calibration dataset $D_{\mathrm{cal}}$ and a
    desired index value $K$, the probability that a data point randomly sampled from a trajectory in the empirical distribution has a nonconformity score less than that of the $K$th point in $D_{\mathrm{cal}}$ is given by:
    \begin{align} \label{eq:expected_threshold}
        P(\rho_{t+1} \leq \rho^{(K)}) = \frac{K}{N+1},
    \end{align}
    provided the samples are exchangeable and assuming randomness over $\rho^{(i)}$,~${\forall i}$, (see Theorem D.1 in \citep{angelopoulos2021gentle}).
    This corresponds to our desired failure probability $\delta$, as
    \begin{align}
        P(\rho_{t+1} \leq \rho^{(\lceil(N+1)(1-\delta)\rceil)}) = \frac{\lceil(N+1)(1-\delta)\rceil}{N+1} \geq (1-\delta).
    \end{align}
    Further, we set the threshold $C$ using the value of $\rho^{(K)}$: any value of $C$ such that $C \geq \rho^{(K)}$ preserves the inequality:
    \begin{align}
        P(\rho_{t+1} \leq C) \geq (1-\delta).
    \end{align}
    
    \begin{remark}
        We note that although $\delta$ can be freely set by the user, in practice the resolution of increments in confidence in the associated coverage, as well as its maximum and minimum value, are determined by the size of the calibration set.
    \end{remark}
    
    The MPC control strategy is determined by the output of the OOD detector, with the value ${B = 0}$ corresponding to nominal behavior and ${B = 1}$ corresponding to OOD behavior. In the nominal setting, our controller utilizes the nominal MPC control scheme \emph{MPC I} and the conservative MPC control scheme \emph{MPC II} in the OOD setting. We apply the control inputs computed by the controller for $H$ time steps, before resolving the corresponding MPC problems over a horizon of $T$ time steps.

    \section{Simulation Setup}
    \label{app:sim_appendix}

    \subsection{Predictor Architecture}
    The neural network ensemble used for OOD detection and trajectory prediction consists of $n=10$ identical sequential multi-layer perceptrons (MLPs). Each MLP, the structure of which is visualized in Figure \ref{fig:predictor_architecture}, is comprised of two dense hidden layers with $32$ neurons and an output layer with $2$ neurons. We use the ReLU (rectified linear unit) activation function after each hidden layer. The input to each MLP consists of the positions of the pedestrian over $\ell=14$ consecutive time steps, while each MLP outputs the $2$D position of the pedestrian at the next time step.

    \subsection{Dataset}
    \label{app:sim_appendix_dataset}
    We use the Vehicle-Crowd Intraction[sic] (VCI) dataset \citep{yang2019top} to train and evaluate the performance of each method in simulation. The dataset used, a subset of the total VCI dataset, consists of $110$ trajectories of pedestrians crossing a roadway, with each trajectory comprised of the position data of a pedestrian for $154$ time steps. In our evaluations, we randomly initialize the position of the pedestrian, relative to the ego vehicle. The horizontal starting positions are sampled from a normal distribution with a mean $40$ meters from the horizontal starting position of the ego vehicle and a standard deviation of $2.5$ meters. The starting vertical position of the pedestrian is determined by the direction of travel in the specific trajectory obtained from the dataset. To generate the training data, we process the dataset into input-output data pairs, where the input consists of the positions of the pedestrian over $14$ consecutive time steps and the output consists of the position of the pedestrian at the following time step.
    
    Of the 110 trajectories, 10 are selected randomly to serve as the \emph{test dataset} $D_{\mathrm{test}}$ and used in evaluating the performance of each algorithm. For the remaining 100 trajectories, one data pair from each trajectory is selected uniformly randomly and removed, and these 100 data points form the \textit{calibration dataset} $D_{\mathrm{cal}}$ used in conformal prediction. The remaining data points comprise the \emph{training dataset} $D_{\mathrm{train}}$, used in training the parameters of each neural network in the ensemble through Adam optimization. These neural networks were used in testing for nearly all of the methods and benchmarks, with the exception of the testing for benchmarks based on \citep{lindemann2023safe} and \citep{muthali2023multi}. For the benchmark based on \citep{lindemann2023safe}, 20 uniformly randomly selected trajectories from the 100 full trajectories were used to form the calibration dataset, with the remaining 80 trajectories used in training the neural networks. For the benchmark based on \citep{muthali2023multi}, no calibration dataset was formed, and all input-output data points from the 100 trajectories were used for training.
    
    To account for an imbalance in the training dataset between the number of trajectories where the pedestrian moves upwards and the pedestrian moves downwards, trajectories where the pedestrian is moving upwards are vertically reflected to prevent the overall accuracy of the ensemble from changing depending on the vertical direction of the pedestrian's movements.
    We do not perform extensive feature engineering, although the described techniques will work well in cases where feature engineering is included for the inputs to the neural network ensemble.
    
    One of the limitations of the dataset is the amount of data available: with only 15400 data points in total available across training, calibration and test datasets, it is not possible to make statistical determinations that depend on large amounts of data. To remedy this, a neural network more sophisticated than those used for the predictor ensemble was trained on the existing trajectories and used to generate additional pedestrian trajectories, which we used for the calibration process and statistical analysis. The architecture of the network used to generate new data can be found in Figure \ref{fig:synthesizer_architecture}, where the neural network uses an LSTM layer with 128 neurons, a dense layer of 128 neurons and ReLU activations to transform rectilinear trajectories of 2-dimensional states into trajectories of two-dimensional states mimicking the behavior of a pedestrian.

    \subsection{OOD Detector}
    In our simulations, we utilize ${K=97}$, for $N=100$ (i.e., $\delta \approx 3.9604\%$) in calibrating the OOD detector, which yielded good results for reliable OOD detection, with $C$ set to the value of $\rho^{(K)}$ rounded up to the third decimal point. Careful readers will notice that the probabilistic guarantees are specifically placed on the probability of correctly determining that a given data point was generated by the same distribution that generated the training data, but as we will show in the section detailing the results of our experiments, over the course of our experiments we were able to reliably detect changes in the underlying data-generating distribution.

    In our simulations, our OOD detector classifies any data point with an uncertainty score above $C$ as being out-of-distribution. We note, however, that the simplicity of this detection framework can result in incorrect classifications of in-distribution samples as OOD, given the probabilistic guarantees in Theorem~\ref{thm:prob_guarantee}. Nevertheless, we can address this challenge using existing methods such as \citep{bates2023testing}, which enables the detector to ignore occasional high-uncertainty scores from in-distribution points, but still respond appropriately to persistent high-uncertainty scores that indicate a distribution shift. Further, other uncertainty metrics in anomaly detection, such as those in \citep{sharma2021sketching} and \citep{laxhammar2011sequential} can be applied in OOD detection. However, these methods require access to the training and calibration data, unlike our method.

    \subsection{Dynamics Model}
    The dynamics model of the autonomous vehicle is similar to the Reeds-Shepp car \citep{reeds1990optimal}, a car that can move both forwards and backwards as well as turn left and right. We denote the state of the autonomous vehicle by ${\mathbf{x}_t = (x_t^{\intercal},\theta_t,V_t,\kappa_t)^\intercal}$, consisting of its $2$-D position ${x_{t} \in \mathbb{R}^{2}}$, its attitude (heading-angle) ${\theta_{t} \in \mathbb{R}}$, its longitudinal velocity ${V_{t} \in \mbb{R}}$, and the curvature of its trajectory ${\kappa_{t} \in \mbb{R}}$.
    We impose limits on the vehicle's speed with ${\lvert V_{t} \rvert \leq V_{\max}}$,~${\forall t}$, where $V_{\max}$ denotes the maximum attainable speed by the vehicle in any direction. Further, we constrain the curvature of the vehicle's trajectory with ${\lvert \kappa_{t} \rvert \leq \kappa_{\max}}$,~${\forall t}$, where $\kappa_{\max}$ denotes the maximum allowable curvature.
    The vehicle's control input at time $t$, denoted by ${\mathbf{u}_{t} \in \mbb{R}^{2}}$, comprises of its acceleration ${a_{t} \in \mbb{R}}$ and the pinch (the first derivative of the curvature) of its trajectory ${p_{t} \in \mbb{R}}$. Given the state and control inputs of the autonomous vehicle, we describe its dynamics in discrete time by the model:
    \begin{equation}
        \label{eq:vehicle_dynamics}
        \mathbf{x}_{t+1} = \mathbf{x}_t + h \begin{bmatrix} 
        V_t \cos \theta_t \\
        V_t \sin \theta_t \\
        V_t \kappa_t \\
        a_t \\
        p_t \\
        \end{bmatrix}.
    \end{equation}
    The dynamics model in (\ref{eq:vehicle_dynamics}) results from applying Euler's Forward Integration to approximate the solution to the ordinary differential equations describing the dynamics of the car. In Figure~\ref{fig:vehicle_state_diagram}, we show the state variables of the autonomous vehicle.
    
    \begin{figure}[th]%
        \begin{minipage}[t]{0.46\textwidth}
            \centering
            \includegraphics[width=\linewidth]{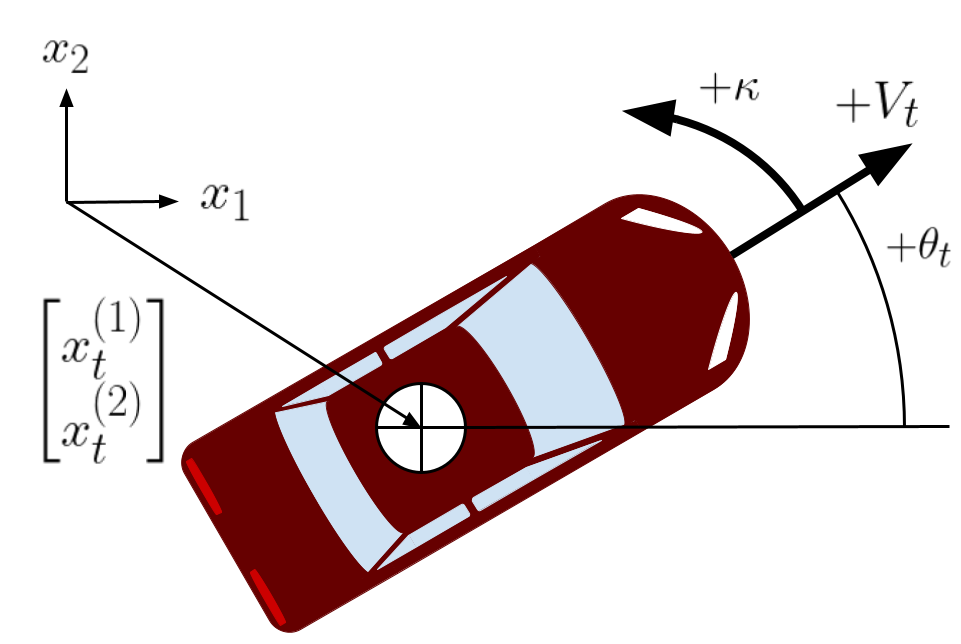}
            \captionof{figure}{State of the Reeds-Shepp Vehicle, with its position ${x_{t} \in \mathbb{R}^{2}}$, its speed ${V_{t} \in \mbb{R}}$, its heading-angle ${\theta_{t} \in \mathbb{R}}$, and curvature ${\kappa_{t} \in \mbb{R}}$.}
            \label{fig:vehicle_state_diagram}
        \end{minipage}
        \hfill
        \begin{minipage}[t]{0.46\textwidth}
            \centering
            \includegraphics[width=\linewidth]{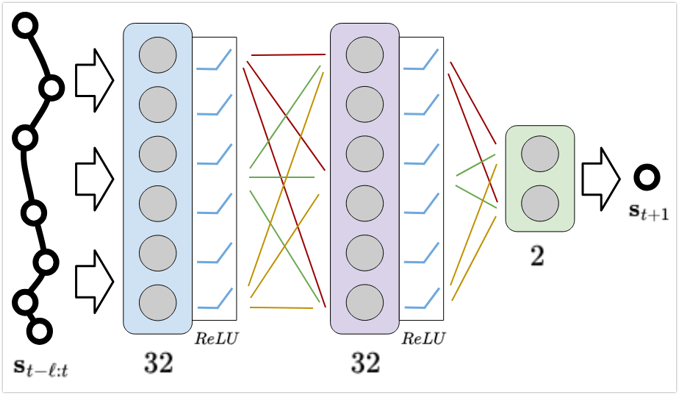}
            \captionof{figure}{Architecture of each predictor network $f_{a}^{i}$, comprising the neural network ensemble for trajectory prediction of pedestrians.}
            \label{fig:predictor_architecture}
        \end{minipage}
    \end{figure}

    \subsection{Trajectory Prediction}
    To generate predictions at each time step for the movements of a pedestrian for the time remaining in the simulation, we recursively apply predictions for the number of time steps remaining in the simulation. Figure~\ref{fig:predictor_architecture} shows the architecture of each predictor. We note that the individual models comprising the ensemble are not required to have the same architecture. After making a prediction at time step $t$ of the pedestrian's position at time step $t+1$ (using the pedestrian's positions during the preceding $\ell$ time steps), we query the network ensemble again and define the input to the network by combining the positions corresponding to the preceding $\ell-1$ time steps and the prediction made for $t+1$ in order to get a prediction for the pedestrian position at $t+2$; we subsequently make predictions further into the future in this way until we reach the number of remaining time steps in the simulation. To predict pedestrian movements for the initial time steps (before enough time steps have elapsed in a simulation for the network ensemble to make predictions), we assume that the pedestrian moves at $1.1m/s$, a speed close to the mean pedestrian speed recorded in the trajectories of the training set, with the direction of movement given by the starting position of the pedestrian.

    This particular neural network design and implementation of multi-step predictions was selected to minimize the complexity of the neural network architecture. While the method of prediction is less sophisticated and computationally efficient than many state-of-the-art implementations, the reduced complexity allows us to emphasize that the principles of neural networks exploited to detect out-of-distribution behavior work in small-scale implementations and are not emergent behavior of more sophisticated models.

    \subsection{Model Predictive Control}
    The MPC optimizer uses an implementation of sequential convex programming, including affine approximations of nonconvex constraints. Starting with an initial estimate for a trajectory that does not incorporate physical constraints, the optimization problem is run multiple times until the solution converges. The initial estimate used depends on the control strategy. When the control strategy is based on neural network predictions, the estimate first consists of linear interpolations of the starting and goal states, and then, of a trajectory that avoids obstacles but ignores constraints imposed by control limits. For the conservative control strategy, the initial estimate consists of repeated instances of the starting state. Once an initial estimate is produced, the MPC optimizer computes full, distinct trajectories of states and control inputs to reach as close to the target goal as possible by iteratively modifying the solution until convergence is reached. In this particular implementation, the solver resolves the problem every $H=5$ time steps, each time calculating the trajectory for the vehicle to follow for the time remaining in the simulation.
    
    \subsection{Simulation Setup}
    The simulation environment is based on a two-lane road, $7.2$ meters wide in total. The goal of the ego vehicle is to move with an initial velocity of $10\mathrm{m/s}$ parallel to the road from an initial position in the middle of the right lane  to a goal position in the middle of the right lane $70\mathrm{m}$ further along the road at the same velocity. In the objective function, we penalize the error between the vehicle's position and the goal position more than the error between its velocity and the goal velocity. Each simulation runs for $150$ time steps at $23.976\mathrm{Hz}$, the frequency at which the original pedestrian data was sampled, meaning each simulation represents approximately $6.256$ seconds in real time.
    
    The dimensions of the autonomous vehicle are based on the dimensions of the 2019 Ford Fiesta \citep{ford2019fiesta}, with a width of $1.8$ meters, a length of $4.0$ meters and a minimum turn radius of approximately $5.913$ meters. We assume the autonomous vehicle has a maximum speed of $20\mathrm{m/s}$, while the pedestrian has a maximum speed of $4.5\mathrm{m/s}$, close to the speed at which an average person can run for short periods of time. The vehicle has two radii for collision and obstacle avoidance: $0.9\mathrm{m}$, based on vehicle width, which is the minimum distance that the vehicle must keep from the upper and lower road boundaries, and $2.0\mathrm{m}$, based on vehicle length, which is the minimum distance that the vehicle must keep from the pedestrian to prevent a collision.
    The pedestrian is represented as a circular object with a radius of $0.5\mathrm{m}$ and is assumed to be able to move in any direction on a planar surface, with the state at time step $t$, $\mathbf{s}_t \in \mbb{R}^2$, consisting of the position of the pedestrian in two dimensional space. In these simulations, we consider a single pedestrian interacting with an autonomous vehicle. However, we note that our proposed algorithm applies to problems with multiple pedestrians, or problems with both pedestrians and other vehicles.
    
    \subsection{Synthetic Data}
    To train the neural network used in generating synthetic data, original pedestrian states were used as example outputs: $\mathbf{s}_{1:T}$ represents all of the states in one trajectory with $T$ time steps. Simultaneously, the corresponding inputs to each trajectory, $\mathbf{r}_{1:T}$, were generated by taking the starting and ending states of the trajectory so that $\mathbf{r}_1 := \mathbf{s}_1$ and $\mathbf{r}_T := \mathbf{s}_T$ and then finding linear interpolations of these states. New, synthetic trajectories were generated after training by creating new rectilinear trajectories to serve as inputs to the network: as average pedestrian speeds and the angles between their starting and ending states each approximated a normal distribution, new rectilinear trajectories were created by sampling from normal distributions following the parameters matching those of the real data and passing the resulting trajectories through the network. The vertical direction then had a 50\% chance of being inverted, and the horizontal starting position of the pedestrian relative to the starting position of the ego vehicle was normally distributed as the trajectories drawn from the VCI dataset were.
    
    \subsection{Pedestrian Behavior}
    For each simulation, the pedestrian adopts one of two modes of behavior for the whole trajectory. For each of the $10$ trajectories in the test set, the pedestrian either follows the selected trajectory throughout while crossing the crosswalk (representing in-distribution behavior for the predictor ensemble), referred to as \emph{nominal} behavior, or the pedestrian can follow one of these trajectories for a set period of time (in these simulations, $1.3$ seconds) and then suddenly change their behavior by continuously moving towards the ego vehicle at maximum speed, referred to as \emph{insurance fraud} behavior, representing a pedestrian attempting to force a collision with the ego vehicle. This presents both a worst-case scenario for the ego vehicle, which must prevent a collision despite the restrictions on its movements and a shift in distribution that the OOD detector must identify and respond to over the course of the simulation.

    \subsection{Statistical Analysis of Safety Assurances on Real and Synthetic Data}
    \label{app:sim_appendix_stats_analysis}
    Here, we analyze the statistical guarantees of our out-of-distribution detector empirically. We note that conformal prediction only provides probabilistic guarantees on the false-positive rate (i.e., on an incorrect identification of a data point in a nominal trajectory as an out-of-distribution sample). To assess the tightness of this guarantee, we apply the out-of-distribution detector at $27$ time steps in each of the $10$ \emph{nominal} simulations, and at the $23$ corresponding time steps after the shift in pedestrian behavior from nominal to OOD behavior in each of the $10$ \textit{insurance fraud} simulations. We summarize these results in Table~\ref{tab:confusion_matrix}, noting that the OOD detector has a false-positive rate of ${4.4\%}$, which is close to the probabilistic error rate of ${4\%}$ used as a calibration setpoint for conformal prediction. In Figure~\ref{fig:false_positive}, we show a simulation where the detector incorrectly classifies a nominal behavior as being OOD. However, our detector is able to recover from incorrect classifications and subsequently classifies the behavior of the pedestrian as being in-distribution.
    
    \begin{table}[th]
    	\centering
    	\caption{Confusion matrix showing the performance of the out-of-distribution (OOD) detector on the pedestrian trajectory data. 
        The calibrated error rate $\delta \approx 4\%$ is close to the empirical false-positive rate
        (shown in bold).}
    	\label{tab:confusion_matrix}
    	\begin{adjustbox}{width=0.6\linewidth}
    		{\begin{tabular}{l c c}
    				\toprule
    				True Behavior & Identified as Nominal & Identified as OOD  \\
    				\midrule
    				Nominal & $258$ ($95.6\%$) & $\mbf{12}$ ($\mbf{4.4\%}$) \\
    				  OOD & $20$ ($8.7\%$) & $210$ ($91.3\%$) \\
    				\bottomrule
    		\end{tabular}}
    	\end{adjustbox}
    \end{table}
    
    \begin{figure}
        \centering
        \includegraphics[width=0.65\linewidth]{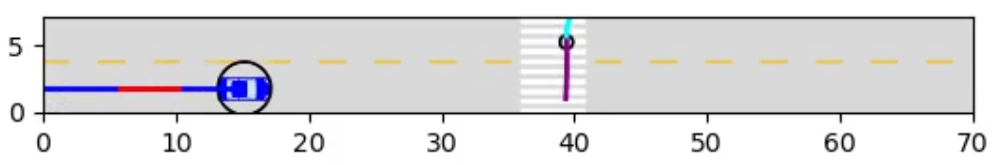}
        \caption{An instance of recovery from misclassification: in this simulation, the detector wrongly classifies nominal behavior as OOD, leading to a temporary switch from the nominal MPC controller (blue trajectory) to the conservative, provably-safe controller (red trajectory), before returning to the correct classification.}
        \label{fig:false_positive}
    \end{figure}
    
    Our OOD detector correctly identifies ${91.3\%}$ of the OOD samples, falsely identifying ${8.7\%}$ OOD samples as nominal samples. In all $10$ ``insurance fraud" simulations, these misclassifications occur during the first two evaluations performed directly after the change in pedestrian behavior, from nominal behavior to attempting to force a collision: after correctly determining that the pedestrian was acting out-of-distribution, our OOD detector does not falsely identify the behavior of the pedestrian as in-distribution.
    
    \begin{figure}%
        \centering
        \begin{minipage}[t]{0.47\textwidth}
            \vspace{0pt}
            \centering
            \includegraphics[width=\linewidth]{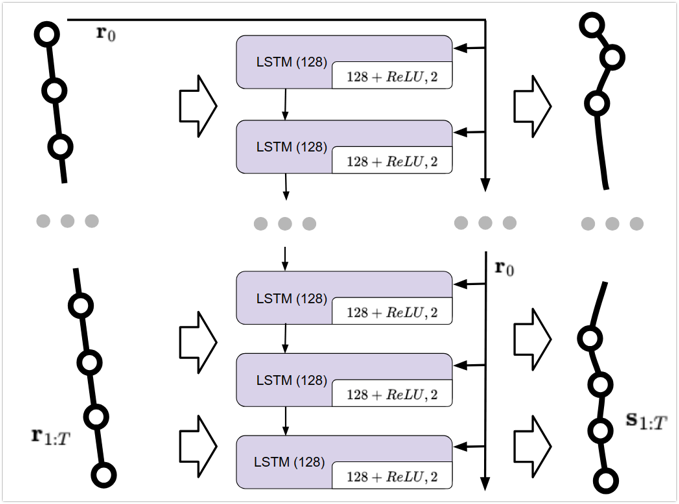}
            \caption{Architecture of the LSTM model used to generate synthetic trajectory data for statistical analysis.}
            \label{fig:synthesizer_architecture}
        \end{minipage}
        \hfill
        \begin{minipage}[t]{0.47\textwidth}
            \vspace{0pt}
            \centering
            \includegraphics[width=\linewidth]{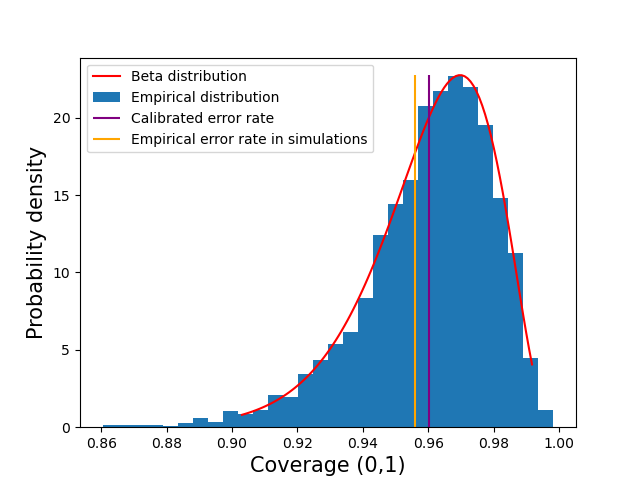}
            \caption{Beta distribution for calibration set of 100 samples and marginal probability of $96\%$, with an empirical coverage rate of $95.6\%$.}
            \label{fig:beta_distribution_empirical}
        \end{minipage}
        
    \end{figure}

    We note that the discrepancy between the observed error rate and the statistical guarantee provided by conformal prediction is expected: more specifically, since each of the trajectories are sampled i.i.d. from the probability distribution $\mathcal{D}$, the distribution of coverage for a threshold calibrated with $C=\rho^{(K)}$ follows a Beta distribution, with shape parameters $(a,b)=(N + 1 - K, K)$ \citep{angelopoulos2021gentle}.
    In Figure~\ref{fig:beta_distribution_empirical}, we show the empirical distribution of coverage, in addition to the associated Beta distribution. We generate the empirical distribution using synthetic trajectories generated by an LSTM (described in Figure~\ref{fig:synthesizer_architecture}), trained on our training dataset, where the trajectories for the calibration and validation datasets are independently sampled without replacement. For each of $3000$ independent trials, $100$ independent trajectories are first generated and used to create a new calibration set of $100$ i.i.d. data points, with a new threshold set equal to the covariance measure associated with ${K=97}$ (the neural networks used to generate the covariance score are the same as during the simulations and were not retrained for this process). Subsequently, the covariance score for each of the data points in a further $150$ independently sampled trajectories is compared to the newly calibrated threshold, with the empirical coverage achieved during each trial being defined as the fraction of the data points in these $150$ trajectories with a covariance score less than or equal to the calibrated threshold. The number of trials that are associated with a given empirical coverage, out of the $3000$ in total, are proportional to the probability density described by the Beta distribution. Figure~\ref{fig:beta_distribution_empirical} shows that the empirical distribution closely matches the Beta distribution.
    
    Note that this implies that different calibration sets sampled from a given dataset drawn from the target probability distribution can result in different levels of coverage in an empirical setting, but in a predictable fashion. In practical terms, this implies that if a given calibration set and ensemble of predictors does not perform as expected with test data or at runtime, it is possible to reshuffle the training and calibration datasets, or to add independent data points from trajectories sampled from $\mathcal{D}$ during runtime to the calibration set, and recalibrate the covariance threshold with the new calibration set. Repeated instances of this process, in a predictable fashion, are increasingly likely to yield an ensemble and threshold that behave with the desired marginal probability characteristics, and theoretical guarantees will remain valid as long as the points used in the calibration set are independent and not included in the training set.

    \subsection{Trajectron++, a Large-Scale Trajectory-Prediction Model}
    \label{app:sim_appendix_trajectron}
    
    The Trajectron++ predictor is multi-modal, with the predicted positions of a given agent (which can be a pedestrian or a vehicle) represented as a Gaussian mixture model (GMM). At each time step, the system produces a set of 25 trajectories for each agent, with each of these trajectories representing a mode and having an assigned probability as part of a categorical distribution (there are several ways that these modes can be sampled from or combined to generate a single predicted trajectory, but this was not used over the course of our experiments). Each of the predicted trajectories corresponding to an individual mode consisted of 12 time steps, with each time step represented by the mean and covariance of a 2D Gaussian distribution.

    The multi-modal nature of a model such as Trajectron++ can be leveraged to detect an increase in uncertainty, most notably through the entropy of the model as a whole or that of its constituent modes. In our case, uncertainty was estimated using the weighted sum of the determinants of the covariance matrices for the first predicted time step in each mode of the model (since the determinant of the covariance matrix has an affine relationship to the entropy of the individual Gaussian distribution):
    
    \begin{align}
        \rho_{t+1} = \sum_{z \in \mathcal{Z}}p_\theta (z|\mathbf{s}) | \Sigma_{t+1}^{(z)} | ,
    \end{align}
    where $z$ is an integer representing a given mode within the model, $p(z|\mathbf{s})$ is the probability of this mode in the categorical distribution at the current time step, $\Sigma_{t+1}^{(z)}$ is the covariance matrix for the Gaussian distribution representing the first predicted time step for the mode and $\mathcal{Z}$ is the set of the integers representing all of the modes in the GMM.
    As these experiments are specifically intended to demonstrate outlier detection with this type of model and take place in cluttered environments with a large variety of interacting agents, the behavior of the ego vehicle follows that described in the original dataset, rather than being guided by MPC. The ego vehicle makes no attempt to dodge pedestrians displaying \textit{insurance fraud} behavior.

    \smallskip
    \noindent\textbf{Dataset.}
    In order to demonstrate the use of conformal prediction to enable the calibration of out-of-distribution detection on a commercial-scale model in a realistic scenario, we evaluate the performance of the detection algorithm in scenarios included in the nuScenes \citep{caesar2020nuscenes} dataset, which is taken from actual recorded vehicle and pedestrian movements in several drives on public roads, sampled at 2Hz. In the dataset, multiple pedestrians and other vehicles are visible simultaneously, and different agents appear in these scenes for different amounts of time. Each scene lasts 20 seconds, with an agent that is visible throughout the scene appearing for 41 time steps.
    Of the 850 annotated scenes available in the nuScenes dataset, which are assumed to be independent, 100 which contained adult pedestrians were removed and placed into a calibration set, with the trajectories used for training and testing drawn from the remaining 750 scenes. After training one instance of the Trajectron++ model, the trained model was used to perform calibration via conformal prediction, for the outlier detection: in each scene in the calibration set, one pedestrian is randomly selected, a time step in the corresponding trajectory is uniformly randomly selected and the uncertainty value is calculated for the prediction at this time step. This gives 100 independent data points used for the calibration procedure.
    Anomalous data was generated by modifying the behavior of adult pedestrians visible for 41 time steps, so that their positions would move, at each time step, towards the current position of the ego vehicle at the same speed as in the \textit{insurance fraud} behavior in the previous experiments (all other data remained unmodified and no collision detection was considered). In total, 230 anomalous trajectories were analyzed in the test procedure, while there were a total of 6230 nominal trajectories of various lengths analyzed.

    We repeated the statistical analysis on trajectories from the nuScenes dataset with Trajectron++, summarized in Table~\ref{tab:confusion_matrix_trajectron}. In Table~\ref{tab:confusion_matrix_trajectron}, the percentage values are computed relative to the ground-truth \emph{nominal} and \emph{OOD} behaviors: nominal cases cover 6230 trajectories of pedestrians in the dataset, while OOD cases use 230 trajectories with shifted pedestrian behavior. The calibrated error rate for individual data points is $\delta \approx 4\%$; a trajectory is classified as OOD when at least one data point in it has a nonconformity measure above the calibrated threshold and nominal otherwise.

    \begin{table}[th]
    	\centering
    	\caption{Confusion matrix showing the performance of the OOD detector through the various modes of the Trajectron++ predictor on trajectories from the nuScenes dataset.}
    	\label{tab:confusion_matrix_trajectron}
    	\begin{adjustbox}{width=0.6\linewidth}
    		{\begin{tabular}{l c c}
    				\toprule
    				True Behavior & Identified as Nominal & Identified as OOD  \\
    				\midrule
    				Nominal & $5361$ ($86.1\%$) & $869$ ($13.9\%$) \\
    				  OOD & $19$ ($8.3\%$) & $211$ ($91.7\%$) \\
    				\bottomrule
    		\end{tabular}}
    	\end{adjustbox}
    \end{table}

    \subsection{Statistical Analysis for Data Collection}
    As mentioned in Section \ref{Simulations}, since each of the trajectories are sampled i.i.d. from the relevant probability distribution, $\mathcal{D}$, the distribution of coverage for a threshold calibrated with conformal prediction follows a Beta distribution; this can be used to inform the data-gathering process when collecting data for a new application. Although the coverage value for an outlier detection threshold calibrated using our conformal method is the same as expected for any valid calibration set, the empirical coverage value attained by an instance of the calibration method is subject to a probability distribution with a shape that depends on the size of the calibration set.
    
    For a valid calibration set with $N$ data points (in sorted order) and a calibration threshold set using the nonconformity measure of the $K$th data point, such that $C :=\rho^{(K)}$, the expected coverage of the empirical calibration threshold is determined by $N$ and $K$. With $\delta_e$ used to represent the empirical equivalent to the ideal false positive detection rate, $\delta$, used to guide the calibration process:
    \begin{align}
        \mathbb{E}[1-\delta_e] = \frac{K}{N+1}.
    \end{align}
    With the value of $K$ set by $N$ and the desired coverage rate such that $K := \lceil(1-\delta)(N+1) \rceil$, the probability of $\delta_e$ resolving close to its ideal value, $\delta$, increases as the value of $N$ increases. Given that the empirical coverage follows a Beta distribution, as described in Theorem \ref{thm:prob_guarantee} and seen in Figure \ref{fig:beta_distribution_empirical}, we can use Beta functions and the corresponding shape parameters $(a,b)=(K,N+1-K)$ to calculate the probability of $\delta_e$ landing within a desired range of values. Specifically, we can use the Beta function, $B(a,b)$, and Incomplete Beta function evaluated at $x$, $I_x(a,b)$, to calculate the probability of $(1-\delta_e)$ resolving between $x_1$ and $x_2$ for a given $N$ and $K$:
    \begin{align}
        P(x_1 \leq 1-\delta_e \leq x_2) = \frac{I_{x_2}(a,b)-I_{x_1}(a,b)}{B(a,b)}.
    \end{align}
    (Note that this is equivalent to subtracting two instances of the cumulative distribution function for the Beta distribution.) For example, for the values $N=1000$ and $K=961$ (corresponding to $(1-\delta) \approx 0.96$), the probability of $(1-\delta_e)$ resolving between $0.95$ and $0.97$ after a single instance of the training and calibration process is approximately $89.65\%$. The procedure for calculating this probability for a calibration dataset containing $N_{test}$ data points is described in Algorithm \ref{alg:calculate_probability}.
    
    \begin{algorithm2e}
    	\KwIn{Dataset Test Size $N_{test}$, Failure Probability $\delta$ and Bounds $x_1$ and $x_2$}
    	\KwOut{Corresponding Probability $P_{test}$}
    	$K_{test} \gets \lceil (1-\delta)(1 + N_{test}) \rceil$
        
        $a \gets K_{test}$
        
        $b \gets N_{test}+1-K_{test}$
    
        $P_{test} \gets (I_{x_2}(a,b)-I_{x_1}(a,b))/B(a,b)$
        
    	\Return $P_{test}$
    \caption{CalculateProbability($N_{test},x_1,x_2,\delta$)} \label{alg:calculate_probability}
    \end{algorithm2e}
    
    This principle can also be used before any data is collected, to estimate the number of calibration data points required to ensure that the empirical coverage value falls within specified bounds around the desired coverage, with a user-specified confidence (i.e., the size of the dataset required for $(1-\delta_e)$ to be within a certain distance of $(1-\delta)$, with a desired probability, after a single iteration of the calibration procedure). A corresponding value of $N$ can be identified using a bisection search method specially adapted for this problem, described in Algorithm \ref{alg:bisection_search_N}.
    
    \begin{algorithm2e} [th]
        \caption{Bisection Search to Estimate Required Calibration Dataset Size}
        \label{alg:bisection_search_N}
    
        \SetKwRepeat{doparallel}{do in parallel}{while}
    
        \KwIn{Failure Probability $\delta$, Desired Probability $P$, Desired Precision $p$ and Bounds $x_1$ and $x_2$}
        \KwOut{Estimated Dataset Size $N$}
    
        Initialize test values $N_{lo}$, $N_{hi}$.
    
        $P_{lo} \gets $CalculateProbability($N_{lo},x_1,x_2,\delta$)
    
        $P_{hi} \gets $CalculateProbability($N_{hi},x_1,x_2,\delta$)
        
        \While{$|N_{hi} - N_{lo}| > p$} {
            
            \uIf{$P_{hi} < P$}{
                $N_{hi} \gets 2N_{hi}$
                
                $P_{hi} \gets $CalculateProbability($N_{hi},x_1,x_2,\delta$)
            }
            \uElseIf{$P_{lo} > P$}{
                $N_{lo} \gets \lceil N_{lo}/2 \rceil$
                
                $P_{lo} \gets $CalculateProbability($N_{lo},x_1,x_2,\delta$)
            }
            \uElse{
                $N_{test} \gets \lceil (N_{lo} + N_{hi})/2 \rceil $
                
                $P_{test} \gets $CalculateProbability($N_{test},x_1,x_2,\delta$)
                
                \uIf{ $P_{test} < P$} {
                $N_{lo} \gets N_{test}$
                
                $P_{lo} \gets P_{test}$
                }
                \uElse{
                $N_{hi} \gets N_{test}$
                
                $P_{hi} \gets P_{test}$
                }
            }
        }
        $N \gets \lceil (N_{lo}+N_{hi})/2\rceil$
        
        \Return $N$
    \end{algorithm2e}
    Note there are some conditions to ensure Algorithm \ref{alg:bisection_search_N} works as intended. The desired coverage value must be between the bounds $x_1,x_2$ such that $0 \leq x_1 < (1-\delta) < x_2 \leq 1$, and $p \geq 1$, otherwise the search may never terminate. (Note it is possible for multiple values of $N$ to provide a valid solution, depending on the search parameters, but when multiple solutions do exist, they are in close proximity to each other.) A convenient starting value for the $N_{lo}$ search parameter is
    \begin{align}
        N_{lo} = \left\lceil \frac{2-\delta}{\delta} \right \rceil,
    \end{align}
    which guarantees that $N_{lo}$ is large enough such that its corresponding $K$ is less than $N+1$ for the requested $\delta$ value, and a convenient starting value for $N_{hi}$ is
    \begin{align}
        N_{hi} = N_{lo} + \left\lceil \frac{2}{1-\delta} \right\rceil,
    \end{align}
    which guarantees that the corresponding $K$ value is different than its equivalent for $N_{lo}$.
\end{document}